\documentclass{article}

\PassOptionsToPackage{numbers, sort, compress}{natbib}
  \usepackage[preprint]{neurips_2019}

\usepackage[utf8]{inputenc} 
\usepackage[T1]{fontenc}    
\usepackage{hyperref}       
\hypersetup{
     colorlinks   = true,
     allcolors    = black
}
\usepackage{url}            
\usepackage{booktabs}       
\usepackage{amsfonts}       
\usepackage{nicefrac}       
\usepackage{microtype}      
\usepackage{algorithmic}

\RequirePackage{amsthm, mathtools, amsmath, amssymb, bm, tikz, tabu}
\usepackage[ruled]{algorithm2e}
\usepackage{graphicx}
\usepackage[bottom]{footmisc}
\usepackage{subfig, lipsum}

\usepackage{enumitem,xcolor}
\newcommand{\oracle}{\text{best response oracle }}
\newcommand\define{\;\stackrel{\mathclap{\normalfont\scriptsize\mbox{def}}}{=}\;}

\DeclareMathOperator*{\argmax}{arg\,max}
\DeclareMathOperator*{\sign}{sign}
\newcommand{\<}{\langle}
\renewcommand{\>}{\rangle}
\newcommand{\R}{\mathbf{R}}

\newcommand{\C}{\mathcal{C}}
\newcommand{\T}{\mathcal{T}}
\newcommand{\E}{\mathop{\mathbb{E}}}

\newcommand{\q}{\mathbf{q}} 
\newcommand{\p}{\mathbf{p}} 

\newcommand{\lz}{\ell_{\text{0-1}}}
\newcommand{\lr}{\ell_r}

\newcommand{\lut}{\ell_{ut}}
\newcommand{\Mz}{M_{0\text{-}1}}

\newcommand{\Mut}{M_{ut}}

\newtheorem{innercustomgeneric}{\customgenericname}
\providecommand{\customgenericname}{}
\newcommand{\newcustomtheorem}[2]{%
  \newenvironment{#1}[1]
  {%
   \renewcommand\customgenericname{#2}%
   \renewcommand\theinnercustomgeneric{##1}%
   \innercustomgeneric
  }
  {\endinnercustomgeneric}
}

\newcustomtheorem{customtheorem}{Theorem}
\newcustomtheorem{customlemma}{Lemma}
\newcustomtheorem{customproposition}{Proposition}
\newcustomtheorem{customcorollary}{Corollary}
\usepackage{printlen}

\newtheorem{theorem}{Theorem}
\newtheorem{lemma}{Lemma}

\newtheorem{corollary}{Corollary}

\usepackage{hyperref}
\hypersetup{colorlinks=true,linkcolor=blue,citecolor=blue,urlcolor=blue}

\title{Robust Attacks against Multiple Classifiers}

\author{%
  Juan C. Perdomo \\
  University of California, Berkeley\\
  \texttt{jcperdomo@berkeley.edu} \\
 \And
 Yaron Singer \\
 Harvard University\\
 \texttt{yaron@seas.harvard.edu} \\
 }

\begin{document}
\maketitle
\vspace{-13pt}
\begin{abstract}
We address the challenge of designing optimal adversarial noise algorithms for settings where a learner has access to multiple classifiers. We demonstrate how this problem can be framed as finding strategies at equilibrium in a two-player, zero-sum game between a learner and an adversary. In doing so, we illustrate the need for randomization in adversarial attacks. In order to compute Nash equilibrium, our main technical focus is on the design of best response oracles that can then be implemented within a Multiplicative Weights Update framework to boost deterministic perturbations against a set of models into optimal mixed strategies. We demonstrate the practical effectiveness of our approach on a series of image classification tasks using both linear classifiers and deep neural networks.
\end{abstract}
\vspace{-5pt}
\section{Introduction}
\label{sec:intro}

In recent years there has been a growing concern regarding the sensitivity of learning algorithms to noise and their general stability. State-of-the-art classifiers that achieve or even surpass human level performance can be reliably fooled by perturbing inputs with an imperceptible amount of noise~\cite{szegedy, fgm, deepfool, momentummethod}. To evaluate classifiers' robustness and improve their reliability, the study of adversarial noise has become a central focus in machine learning~\cite{datapoisoningcertified, KL17,raghunathan2018certified,yosinki,ilyas,physicalworld}.

One of the most active areas of research within adversarial noise, has been the design of adversarial attacks against a \emph{single classifier} (e.g.~\citep{szegedy, fgm, deepfool, universaladversarial, momentummethod, carlini}).  Given a data point and a classifier, the goal of these algorithms is to find the perturbation of minimum norm that, when added to the data, induces the classifier to make the wrong prediction.  Adversarial attacks have gained a great deal of attention as they inform the design of robust models and test the robustness of existing models.

A common strategy for robust classification is to randomize decisions across multiple classifiers.  This approach is used in \emph{gradient boosting}~\cite{F00,F02} as well as for boosting linear classifiers~\cite{BHKM15} and neural networks~\cite{robertpaper}.  An attack designed to fool a single model as in~\citep{szegedy, fgm, deepfool, universaladversarial, momentummethod, carlini} may do poorly on another model and therefore randomizing between models is a reasonable defense strategy. 

 
In this paper, we study adversarial attacks against a learner that randomizes between multiple classifiers.  In particular, we consider the design of \emph{provably optimal} adversarial attacks against a set of classifiers. Relative to attacks against a single classifier, characterizing the optimal attack against multiple classifiers is significantly more challenging. Classifiers that achieve the same accuracy can have drastically different decision boundaries. Hence, a perturbation that fools one model may be completely ineffective on another (See example in Figure \ref{fig:randomization}). Heuristically, one may design an attack on the \emph{average} of multiple classifiers, yet such an attack does may be arbitrarily ineffective. We therefore consider a robust optimization approach and define an optimal attack against a set of classifiers as noise that minimizes the maximum accuracy of \emph{all} classifiers in that set.

\begin{figure*}[t!]
\begin{center}
\includegraphics[width=\textwidth]{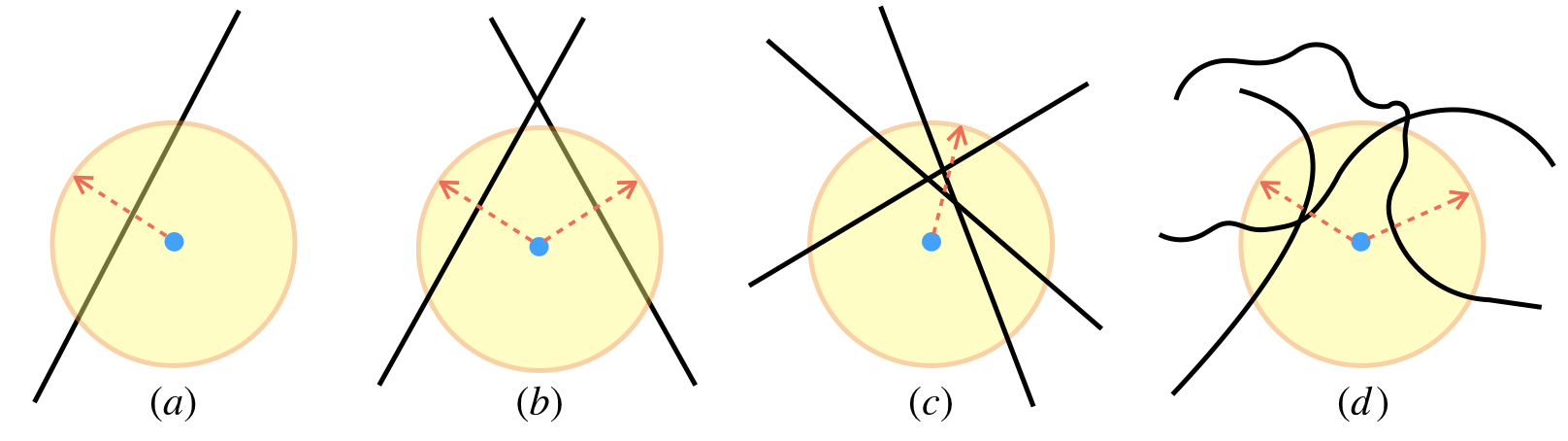}
\end{center}
\caption{Starting on the left, the optimal attack against a single classifier consists of finding a noise vector that pushes the point past the decision boundary. In the case of multiple classifiers, while we can sometimes find perturbations that fool all classifiers simultaneously as seen in $(c)$, attacks that fool one classifier may be ineffective on others. Hence, we must randomize across noise vectors to generate robust attacks as in $(b)$. In $(d)$, we see how this phenomenon extends to nonlinear classifiers.}
\label{fig:randomization}
\vspace{-15pt}
\end{figure*}

We present a principled approach for attacking a set of classifiers which proves to be highly effective for both linear models and deep neural networks. We show that constructing optimal attacks against multiple classifiers is equivalent to finding strategies at equilibrium in a zero-sum game between a learner who selects classifiers from a set, and an adversary that adds (bounded) noise to data points. 

\textbf{Contributions.} It is well known that the Nash equilibria of zero-sum games can be efficiently obtained by applying the celebrated Multiplicative Weights Update algorithm, \emph{if there exists an oracle that computes a best response} to a randomized strategy ~\cite{FreundSchapire, FreundSchapireGamePlaying, kale}. The main technical challenge we address pertains to the design and implementation of these best response oracles in the context of adversarial noise. More generally, our contributions are as follows:
\begin{itemize}[nolistsep]
\item We show how computing optimal attacks against multiple classifiers can be reduced to the task of designing best response oracles within the MWU framework (Section \ref{sec:model});

\item We provide a geometric characterization of the optimization problem of finding best responses for a general set of classifiers and demonstrate how to construct an exact oracle via convex optimization for the case when the set consists of linear classifiers (Section \ref{sec:characterization});

\item We show that when the number of classifiers is super-constant in the dimension of the data, constructing an optimal attack is in general NP-hard even for binary linear classifiers.  To address this, we design a convex relaxation and prove it is guaranteed to return a solution arbitrarily close to optimal in polynomial time under natural conditions (Section \ref{sec:at_scale});

\item We generalize our approach to neural networks (Section \ref{sec:deep_learning}) and empirically validate the efficacy of this principled approach via a series of image classification experiments. We demonstrate how there is a large gap in performance between our methods and existing state-of-the-art approaches when attacking a set of classifiers (Section  \ref{sec:experiments}).
\end{itemize}

Similar to most work on adversarial noise (e.g. ~\citep{szegedy, fgm, deepfool, universaladversarial, momentummethod, carlini}), we assume that the attacker has full knowledge of the classifiers being attacked.  From the perspective of informing the design of robust classifiers, this is an important assumption to make if we wish to certify the robustness of a learning strategy against a worst case adversary.  Practically speaking, this is a sensible consideration since practitioners often use off-the-shelf classifiers that are accessible to attackers.  For many sophisticated machine learning tasks, training state-of-the-art classifiers, such as neural networks, is an expensive endeavor which requires significant experience and tuning. Instead of training their own models, it has become increasingly common for practitioners to make use of pre-trained classifiers that are available online. Most importantly however, the design of optimal attacks against a set of classifiers is a basic theoretical question in robust machine learning that so far has not been addressed. 

\textbf{Related Work.} Our work builds upon a rich literature of adversarial noise and robust optimization in machine learning \citep{szegedy, adversarialgeneralization, computationalconstraints, madry2018towards, shaishalev, duchi}. More specifically, our results on designing attacks that robustly optimize over a set of classifiers can be seen as complementary to that of \citet{robertpaper}. In their paper, they demonstrate how given a  Bayesian oracle that returns an $\varepsilon$-approximate solution for distributions over a set of objectives, one can then compute a distribution over solutions that is $\varepsilon$-approximate in the worst case.  While they abstract away the optimization problem into the existence of an oracle, our main focus is precisely the design of efficient algorithms to compute best responses in the context of adversarial noise. Within the security literature, \citet{song} motivate the use of ensembles as a heuristic approach to attack deep learning classifiers in the context of black box attacks. Yet, the problem setting they consider is significantly different since they are interested in attacking a single unknown classifier rather that robustly optimizing over an entire set.

\section{Optimal Attacks against a Set of Classifiers are a Zero-Sum Game}
\label{sec:model}

To simplify our presentation, we describe the optimal attack on a \emph{single} point $(x,y) \in \R^d \times [k]$, where $\R^d$ is the input space and $[k]$ is the set of labels. This is without loss of generality since each point in a data set $\{(x_j,y_j)\}_{j=1}^m$ can be perturbed independently, and the optimal attack on $\{(x_j,y_j)\}_{j=1}^m$ consists of the $m$ optimal attacks on each point individually.
 
A \emph{deterministic adversarial attack} is a single vector $v \in \R^d$. A distribution $\q$ is a \emph{randomized adversarial attack} if $\q$ is a probability distribution over a set of deterministic attacks $\{v_1, \dots, v_t\}$. We say that an attack is $\varepsilon$-\emph{bounded} if all vectors $v_i$ that have non-zero probability under $\q$ have bounded norm, i.e. $||v_i||_p \leq \varepsilon$. We describe the case for the $\ell_2$ norm, however, our results can be easily extended to a variety of norms, including the $\ell_\infty$ norm (see Appendix \ref{sec:ell_infty}).

For a given classifier $c:\R^d \to [k]$, a deterministic adversarial attack $v$ induces misclassification on $(x, y)$ if $c(x + v) \neq y$. Given a finite set of $n$ classifiers $\C$, an \emph{optimal adversarial attack} on a pair $(x,y)$ is a distribution $\q$ over noise vectors that maximizes the minimum 0-1 loss of the models in $\C$:
\begin{equation}
\label{eq:game}
	 \arg\max_\q \min_{c\in \C}  \E_{v \sim \q} \left[ \lz(c, x + v, y) \right]
\end{equation}
This objective describes the optimal adversarial attack because it has the property that the adversary is \emph{indifferent} as to the classifier chosen by the learner. As we later illustrate both empirically and theoretically, designing attacks against classifiers that are chosen uniformly at random, or even the ensemble of all models, provides no guarantees that there exists a classifier in the set which achieves perfect accuracy. Therefore, the optimal attack for an adversary which faces uncertainty as to the classifier ultimately chosen over the learner is to robustly optimize over the entire set of possibilities.

\textbf{Optimal attacks are equilibrium strategies.} Attacking a set of classifiers can be modeled as a zero-sum game between a learner who selects classifiers $c \in C$ and an adversary that chooses noise vectors $v \in \R^d$, where $||v||_2 \leq \varepsilon$. In addition to pure strategies, players can opt to play randomized strategies. The learner can choose a distribution $\p$ over the set $\C$ and the adversary can select an $\varepsilon$-bounded randomized attack $\q$. Randomization is a necessary property of the model since deterministic attacks are limited in their power to induce misclassification across multiple classifiers as seen in Figure \ref{fig:randomization}. As mentioned previously, the game is played over a single example $(x,y)$. We define the payout function of the game $\Mz (\p, \q)$ as the expected 0-1 loss of the learner:\footnote{
Later on, we modify the game so that it is played over other loss functions. Hence, while $\Mz$ denotes the expected loss of the learner under the 0-1 loss, we let $M_\ell$ denote the expected loss of the learner under an arbitrary loss function $\ell$. We overload notation so that the payoff function $M_\ell(\cdot, \cdot)$ accepts distributions $\p,\q$ as well as pure strategies $c\in \C,v\in \R^d$.}
\begin{equation}
\label{eq:payout}
	   \Mz(\p, \q) \define \E_{c \sim \p, v \sim \q} \left[ \lz(c, x + v, y) \right]
\end{equation}
In this presentation, the learner tries to minimize payouts while the adversary maximizes. The Nash equilibrium of the game is a pair of strategies $\p^\star,\q^\star$ such that the following relationship holds:
\begin{equation}
\min_{c \in \C} \Mz(c, \q^\star) = \max_{v \in \R^d} \Mz(\p^\star, v) = \lambda
\end{equation}
\textbf{Computing optimal adversarial attacks via MWU.} Since the seminal result of Freund and Shapire~\cite{FreundSchapire, FreundSchapireGamePlaying}, it is well known that Multiplicative Weight Updates as described in Algorithm \ref{alg:MWU} can be used to efficiently compute equilibrium strategies of zero-sum games \emph{assuming access to a best response oracle} that returns the optimal pure strategy (best deterministic attack) for any distribution over $\C$.  The main focus of our paper is how to design such best response oracles that enable the implementation of MWU and in doing so allow us to compute optimal attacks. 

\begin{equation}
\label{eq:best_response}
	\textsc{best response}(\p, \varepsilon, M_\ell) \define \underset{v \in \R^d, \; ||v||_2 \leq \varepsilon}{\argmax} M_\ell(\p, v)
\end{equation}
The MWU algorithm computes distributions $\p^\star, \q^\star$ that are within $\delta$ of the equilibrium value of the game $\lambda = \min_\p \; \max_\q \; \Mz(\p,\q)$ using $\mathcal{O}(\frac{\ln n}{\delta^2})$ iterations. \footnote{In practice, the algorithm converges in far fewer iterations as we show through our experiments in  Section \ref{sec:experiments}. We analyze the convergence of the MWU algorithm in Appendix~\ref{sec:mwu_convergence}.} In this work, we focus on developing attacks on sets of neural networks and linear models. However, our framework can be used to generate optimal attacks for any domain in which one can approximate a best response.
\makeatletter
\renewcommand{\@algocf@capt@plain}{above}
\makeatother
\begin{algorithm}[t!]
  \caption{Multiplicative Weight Updates for Optimal Noise}
\begin{algorithmic}
   \STATE {\bfseries Input:} $\C=\{c_i\}_{i=i}^n$, point $(x,y)$, parameters $\varepsilon$, $T$, $\beta$, payoff function $M_\ell(\cdot, \cdot)$
   \STATE \textbf{initialize} $\mathbf{p}_1=(\frac{1}{n},\ldots,\frac{1}{n})$;    \FOR{$t=1$ {\bfseries to} $T$}
   \STATE Set $v_t = \textsc{best response}(\p_t, \varepsilon, M_\ell)$
   \STATE Set $\p_{t+1}[i] \propto \p_{t}[i](1 - \beta)^{M_\ell(c_i, v_t)}$ for every $i \in [n]$
   \ENDFOR
   \STATE {\bfseries Return:} uniform distributions $\p^\star$ over $\{\p_1,\dots, \p_T\}$, $\q^\star$ over $\{v_1, \dots, v_T\}$
\end{algorithmic}
   \label{alg:MWU}
\end{algorithm}

\section{Designing Best Response Oracles for Adversarial Attacks}
\vspace{-5pt}
\label{sec:characterization}
In this section we present our main technical results: the characterization and implementation of best response oracles for adversarial noise. We begin by characterizing the optimization problem for a general set of classifiers under the 0-1 loss. Then, we provide a more refined analysis of the underlying geometry for the case where the set $\C$ is composed of linear models. This refined analysis is the fundamental insight that guides the design of  algorithms for optimal attacks.  Lastly, we present the central result of the section, the existence of an exact best response oracle for linear classifiers.

\textbf{Geometry of best response oracles.} The key observation that allows for the design of best response oracles is that the optimization problem described in Equation \ref{eq:best_response} can be solved by searching over finitely many regions. When the learner selects a distribution over a finite set of classifiers, her loss $\Mz(\p,\cdot)$ can only assume finitely many values, each of which is associated with a particular region $T_j \subset \R^d$. Finding the optimal response then consists of finding points in each region and choosing the one with the highest associated loss. 
\begin{lemma}
\label{lemma:characterization}
Given a point $(x,y) \in \R^d \times [k]$, selecting a distribution $\p$ over a set $\C$ of $n$ classifiers partitions the input space $\R^d$ into $k^n$ disjoint sets $T_j$ such that:
\begin{enumerate}[nolistsep]
	\item For each $T_j$, there exists a unique label vector $s_j \in [k]^n$ such that for all $v$ with $x + v \in T_j$ and $c_i \in \C$, $c_i(x + v) = s_{j,i}$, where $s_{j,i}$ is a particular label in $[k]$.
	\item There exists a finite set of numbers $a_1, \dots a_{k^n}$, not necessarily all unique, such that $\Mz(\p,v) = a_j$ for all noise vectors $v$ such that $x+v \in T_j$.
\end{enumerate}
\end{lemma}
\begin{proof} We define each set $T_j$ according to the predictions of the classifiers $c_i \in \C$ on the perturbed points $x + v$ that lie within $T_j$. In particular, each region $T_j$ is associated with a unique label vector $s_j \in [k]^n$ such that  $c_i(x + v) = s_{j,i}$ for all $c_i \in \C$. This relationship defines a bijection between sets $T_j$ and label vectors $s_j$. Since the predictions of each model are the same for all points in a particular set, the expected loss  $\Mz(\p,v)$=$\sum_{i \in [n]} \p[i] \lz(c_i, x + v, y)$ is constant for all points $x+v$ that lie in that set. Since there are finitely many regions, the loss can only assume finitely many values.
\end{proof}
\vspace{-10pt}
Lemma~\ref{lemma:characterization} shows that if we can design an algorithm to find points in each region $T_j$, we can compute a best response. Next, we show that the design of such an algorithm relies crucially on the geometry induced by the classifiers we wish to attack. In particular, when the set of classifiers consists of linear models, the regions $T_j$ are not only finite, they have the added benefit of also being \emph{convex}.
\begin{lemma}
\label{lemma:convex_sets}(Informal)
Given a point $(x,y) \in \R^d \times [k]$ selecting a distribution $\p$ over a set $\C$ of $n$ linear classifiers, partitions the input space $\R^d$ into $k^n$ disjoint and convex sets $T_j$.
\end{lemma}
This insight allows us to compute best responses via a reduction to convex programming as we present in Theorem \ref{theorem:multi_oracle}. The proof, presented in the Appendix, introduces the construction of an exact algorithm to find points in each region $T_j$. Having found a vector associated with each region, computing a best response then amounts to selecting the perturbation associated with the highest loss. 
\begin{theorem}
\label{theorem:multi_oracle}
	For linear classifiers, implementing an exact \oracle reduces to the problem of minimizing a quadratic function over a set of $k^n$ convex polytopes.
\end{theorem}
\vspace{-15pt}
\section{Computing Optimal Attacks Efficiently}
\vspace{-3pt}
\label{sec:at_scale}
The previous section introduced the key geometric insights necessary to compute best responses. We now analyze the complexity of computing an optimal attack on a set of classifiers in various settings. We show that computing the optimal attack can be done efficiently when the number of classifiers is constant in the input dimension.  In general, however, we show that the problem of designing optimal attacks is NP-hard. We address this challenge by developing a novel algorithmic approach that is based on convex relaxations and is guaranteed to return the optimal solution under natural conditions.

\textbf{Computing optimal attacks efficiently.} For sufficiently rich data distributions, only a small number of classifiers can perform reasonably well.  This assumption is typical within the adversarial examples literature \citep{song, momentummethod, specialists, adversarialtraining, ensembledefenses} where most of the settings considered suppose that the learner has access only to a small \emph{constant} number of classifiers (e.g less than 5). For these settings, when the number of classifiers is constant, a best response is computable in polynomial time.

\begin{corollary}
\label{corollary:constant}
When the number of linear classifiers is constant in the size of the input dimension $d$, the optimal attack on a set of classifiers can be computed in polynomial time.
\end{corollary}
The main idea of the proof is that for a constant number of classifiers $r$, computing a best response as per Theorem \ref{theorem:multi_oracle} requires searching over only polynomially many regions $T_j$. Since MWU takes only polynomially many iterations to converge to a solution, we can compute an optimal attack efficiently. 

\textbf{Hardness of computing best responses.} If the number of classifiers is superconstant, designing an efficient algorithm to compute the optimal best response is NP-hard, even for binary linear models.
\begin{theorem}
Given a set $\C$ of $n$ linear binary classifiers, a number $B$, a point $(x, y)$, a noise budget $\alpha$, and a distribution $\p$, the problem of finding a vector $v$ with $||v||_2 \leq \alpha$ such that the loss of the learner $\Mz(\p, v) = B$ is NP-complete.
\end{theorem}
The proof relies on the geometric characterization developed in Section \ref{sec:characterization} and is deferred to the Appendix. Given the hardness of computing best responses, we now develop an appropriate convex relaxation of the problem and introduce an alternative optimization method using projected gradient descent. Furthermore, we identify a set of natural conditions under which this new approach is guaranteed to return the optimal best response.
\vspace{-5pt}
\subsection{Best Responses via Convex Relaxations}
\vspace{-5pt}

Computing a best response to multiple classifiers is hard when the number of classifiers is super-constant, since for any given $\varepsilon>0$ we can construct instances where no $\varepsilon$-perturbation can succeed in fooling all classifiers.  In such a case, our reduction implies that no known algorithm can do better than exhaustively searching over exponentially-many intersections of decision boundaries.  But if a region where all classifiers are fooled within the noise budget exists -- and our experiments show that this is often the case -- we can efficiently find near optimal solutions using a convex relaxation. 
  

To describe this approach, recall from our characterization in Section \ref{sec:characterization}, that computing a best response is equivalent to searching over a finite number of regions $T_j$, each defined according to the underlying predictions of the learner's classifiers on points in the set. For a given point $(x,y)$ and noise budget $\varepsilon$, we say that a region $T_j$ is a \emph{feasible misclassification set} if there exists a noise vector $v$ s.t. $||v||_2 \leq \varepsilon$ and $x+v \in T_j$ but $c_i(x+v) \neq y$ for all $c_i \in \C$.  Figure \ref{fig:randomization} (c) illustrates a feasible misclassification set -- a region inside the noise budget where all classifiers predict the wrong label. 

Our main algorithmic approach to compute best responses for a super-constant number of classifiers is to apply projected gradient descent to a weighted sum of appropriately chosen loss functions. For binary classifiers $c_i$, with labels in $\{\pm1\}$, predictions are made according to the rule: $c_i(x) = \sign(\<w_i, x\> + b_i)$. Given a point $(x,y)$ and a distribution $\p$ selected by the learner, we attempt to solve the optimization problem outlined in 
\eqref{eq:best_response} by running PGD on a weighted sum of \emph{reverse hinge losses}, $f(v) = \sum_{i=1}^n \p[i] \lr(c_i, x + v, y)$, over $v$ in the $\ell_2$ ball of radius $\varepsilon$. The reverse hinge loss has the property that it is 0 if and only if $x+v$ is misclassified by the classifier $c_i$.\footnote{To preserve the guarantees of MWU, we slightly modify the reverse hinge loss so that its range is constrained to lie in the interval [0,1] by multiplying by a constant. See proof of Theorem \ref{theorem:revhinge} in the Appendix for details.} We now prove that such an approach computes the optimal best response if a feasible misclassification set exists:
\begin{equation}
\ell_{r}(c_i, x + v, y) \define \max \{y( \<w_i, x + v\> + b_i), 0\} 
\end{equation}

\begin{theorem}
\label{theorem:revhinge}
	 For any noise budget $\varepsilon>0$, precision parameter $\delta>0$, and distribution $\p$ over $\C$, running  projected gradient descent for $\mathcal{O}(\varepsilon^2 / \delta^2)$ iterations on $f(v) = \sum_{i=1}^n \p[i] \lr(c_i, x + v, y)$ returns a solution $v_t$ such that $f(v_t) - f(v^\star) \leq \delta$, where $v^\star$ is the global minimum of $f$. Furthermore, if there exists a feasible misclassification set under $\varepsilon$, then it also holds that $f(v_t) - f(v_{BR}) \leq \delta$, where $v_{BR} = \textsc{best response}(\p, \varepsilon, \Mz)$ 
\end{theorem}
In our experiments, we build upon this theoretical result and verify that this method of approximating best responses is effective even in cases where no feasible misclassification set exists. Moreover, given that it returns the optimal solution in the convex case, it serves as a well-principled approach for extending our framework to settings where the optimization problem is nonconvex.
\vspace{-3pt}
\section{Attacking a Set of Neural Networks}
\vspace{-3pt}
\label{sec:deep_learning}
When the learner's set of classifiers consists of neural networks, computing a best response for the adversary still involves a search over finitely many regions $T_j$ as per Lemma 2. However, given their nonlinear decision boundaries, designing an exact algorithm to find the perturbation $v$ of minimum norm such that a point $x+v$ lies in a particular $T_j$ is  intractable since the regions are now nonconvex. 

To compute best responses for this domain, we follow the same pattern as in the previous section and design algorithms by solving a surrogate optimization problem which is computationally efficient and whose analog in the convex setting is guaranteed to be optimal. In particular, we approximate a best response on a set of neural networks by running projected gradient descent on a weighted sum $f(v) = \sum_{i=1}^n \p[i] \lut(c_i, x + v, y)$ of \emph{untargeted reverse hinge losses}, $\lut$.\footnote{We also use this approach to attack linear multiclass models. Given $k$ classes, decisions are made according to $c_i(x) = \argmax_{j \in [k]} c_{i,j}(x)$, where $c_{i,j}(x) = \<w_{i,j}, x\> + b_{i,j}$}  For a neural network $c_{i}$ and input $x$, we define $c_{i,j}(x)$ to the be $j$th logit of the classifier. Predictions are made according to the rule: $c_i(x) = \argmax_j c_{i,j}(x)$.
\begin{equation}
	\lut(c_i, x+v, y) \define  \max\big\{2(\frac{1}{1 + e^{-z}} - .5), 0\big\};  \quad z\define c_{i,y}(x+v)  - \max_{j\neq y}c_{i,j}(x+v)
\end{equation}
We compute attacks on a set of neural networks via Algorithm \ref{alg:MWU} by relaxing the game to be played using the payout function $\Mut$ which we define below. Since the untargeted reverse hinge loss is bounded to [0,1], and is 0 if and only if the point $x+v$ is misclassified, running PGD on a weighted sum of reverse hinge losses can then  be seen as attempting to maximize the function $\Mut(\p,\cdot)$.
\begin{equation}
	\Mut(\p, \q) \define 1 - \E_{v \sim \q, c \sim \p} \left[ \lut(c, x + v, y) \right]
\end{equation}
Although we cannot prove guarantees in the deep learning case, this method directly generalizes ideas we showed were optimal in the linear setting. Like other works that follow this design principle \cite{deepfool, decoupling}, we find that our principled algorithm yields state-of-the-art results in practice.

\vspace{-5pt}
\section{Experiments}
\vspace{-3pt}
\label{sec:experiments}

We evaluate our framework for optimal attacks on a series of image classification tasks. First, we validate our theory for linear classifiers by considering both binary and multiclass experiments on MNIST. Afterwards, we evaluate our approach on ImageNet using deep neural networks. In both cases, we see how our algorithms significantly outperform current methods.

\textbf{Evaluation metrics.} As per our discussion in Section \ref{sec:model}, we evaluate attacks on a set of classifiers according to the  minimum $0\textrm{-}1$ loss (maximum accuracy) they induce across the entire set. Given a deterministic or randomized adversarial attack, $\q$, we measure $\min_{c_i \in C} \Mz(c_i,\q)$. We summarize the strength of a noise algorithm across an entire data set by computing attacks individually for each point and report the average minimum loss of the learner across all points. Lastly, to highlight the difference between fooling classifiers on average and robustly optimizing across all models, we also compute $\Mz(\tilde{\p},\q)$ where $\tilde{\p}$ is the uniform distribution over $\C$.

\textbf{Baselines.} We consider a variety of other methods with which to compare our approach. (i) \emph{Ensemble:} Given a set $\C$, we ensemble $c_i \in \C$ by averaging their predictions and attack the ensemble classifier as a way of generating noise vectors that fool the underlying models; (ii) \emph{Best Individual:} we generate attacks for each model individually, evaluate them on the entire set, and choose the best one; (iii) \emph{Oracle}: we compute a best response to the uniform distribution over classifiers using the oracles we introduce, but do not boost attacks by running MWU for multiple rounds.
\vspace{-3pt}
\subsection{Evaluating Optimal Attacks on Linear Classifiers}
\vspace{-3pt}

\textbf{Experimental setup.} For linear classifiers, we train two sets of 5 linear SVM classifiers on MNIST, one for binary (digits 0 and 1) and another for multiclass (first 3 classes, MNIST 0-2). To ensure that the decision boundaries are sufficiently different, we randomly zero out 75\% of the dimensions of the training set for each classifier. Hence, each model has weight parameters that are nonzero on  a random subset of features. All classifiers achieve test accuracies above 97\%. For our experiments, we randomly select 100 points from each test dataset that are correctly classified by all models.

Using our characterization, we can compute exact margins of each point to the decision boundary (See Appendix for exact computation as well as further details on experimental setup). If we select a noise budget $\varepsilon$ smaller than the minimum margin, then it is impossible to induce any misclassification. If $\varepsilon$ is larger than the max margin, then feasible misclassification sets exist for all points and we are guaranteed to fool all models as per Theorem \ref{theorem:revhinge}. Hence, we select noise budgets in between the min and max margin. In particular, we set $\varepsilon$ to 2.3 and 1.3 for binary and multiclass, respectively. 

\begin{figure}[t!]
\begin{center}
\includegraphics[width=.2\textwidth]{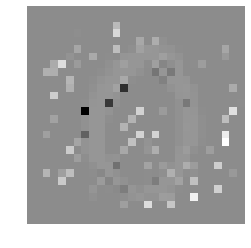} 
\includegraphics[width=.2\textwidth]{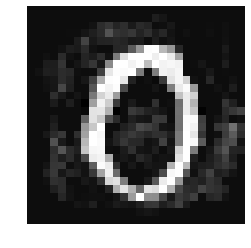} 
\includegraphics[width=.2\textwidth]{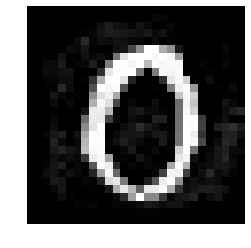} 
\includegraphics[width=.2\textwidth]{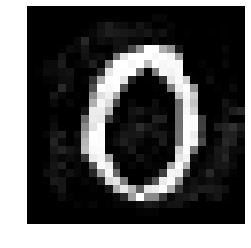} 
\end{center}
\vspace{-13pt}
\caption{Comparison of noise needed to induce similar misclassification across multiclass MNIST classifiers. From left to right: \emph{Best Individual}, \emph{Ensemble}, \emph{Oracle}, \emph{MWU-Oracle}.} 
\label{figure:noise_budget_comparisons}
\vspace{14pt}
\vspace{-31pt}
\end{figure}

We evaluate our framework for optimal noise by running MWU using both the exact oracle described in Theorem \ref{theorem:multi_oracle} (MWU-Oracle) as well as the approximate oracle using PGD on a weighted sum of losses (MWU-PGD). For binary models, MWU-PGD refers to running PGD for 40 iterations on a sum of reverse hinge losses, while for multiclass, PGD is run on untargeted reverse hinge losses. For the case of MWU-PGD, we compute noise solutions by running Algorithm \ref{alg:MWU} on the relaxed version of the game with payoff function $\Mut$ (multiclass) and $M_r$ (binary) as described in Section \ref{sec:deep_learning}.\footnote{We add box constraints to the algorithms to ensure perturbed data points remain valid images in the range [0,1]. See Appendix for further details on the experimental setup and update rule for projected gradient descent.}

To compute the baselines, since the class of linear models is convex, we compute an equal weights ensemble by combining the weight vectors $w_i$ and biases $b_i$ for all $c_i \in \C$  (e.g $w_{ensemble}$$=\frac{1}{n}\sum_{i=1}^n w_i$). We generate noise for the ensemble as well as for each individual model, by computing the theoretically optimal attack against a single linear classifier. This corresponds to the vector returned by the exact best response oracle for when $\C$ consists of a single model, scaled to have $\ell_2$ norm equal to $\varepsilon$. 

\textbf{Results.} We present our results for multiclass in Table \ref{table:multi} and illustrate the convergence of MWU in Figure  \ref{fig:mwu_convergence}. Results for binary classifiers are presented in the Appendix. First, we note that the \emph{Oracle} comparison, MWU-Oracle run for only a single iteration, does significantly better than all other baselines. Since all classifiers have equal weights, this baseline finds the optimal \emph{deterministic} attack, that is the noise vector which induces misclassification across the largest subset of models. 

Second, as seen in Figure \ref{fig:mwu_convergence}, running MWU for several iterations significantly improves the quality of the resulting noise. As per our theoretical analysis, the noise solution at the end of MWU constitutes the optimal \emph{randomized} attack, hence the difference in max accuracy between the first and last round of MWU-Oracle indicates the exact gap between the optimal deterministic and randomized attack.

To further compare across methods, in Figure \ref{figure:noise_budget_comparisons} we present the amount of distortion needed so that each algorithm induces maximum accuracy across classifiers comparable to that induced by MWU-Oracle. Since the classifiers have sparse features, noise generated against any individual model is not effective in fooling other classifiers (entries are nonzero for on mostly disjoint set). Therefore, the \emph{Best Individual} baseline requires significantly more noise than other methods. To match the performance of MWU-Oracle, the \emph{Ensemble} baseline required $60\%$ more noise, while \emph{Oracle} required only a $20\%$ increase in the noise budget. We choose to present this comparison for the MNIST case, since for ImageNet, the noise needed to fool classifiers cannot be visually perceived. 
\vspace{-15pt}
\subsection{Robust Attacks against a Set of Neural Networks}

\begin{figure*}[t!]
\centering

\includegraphics[width=0.48\linewidth]{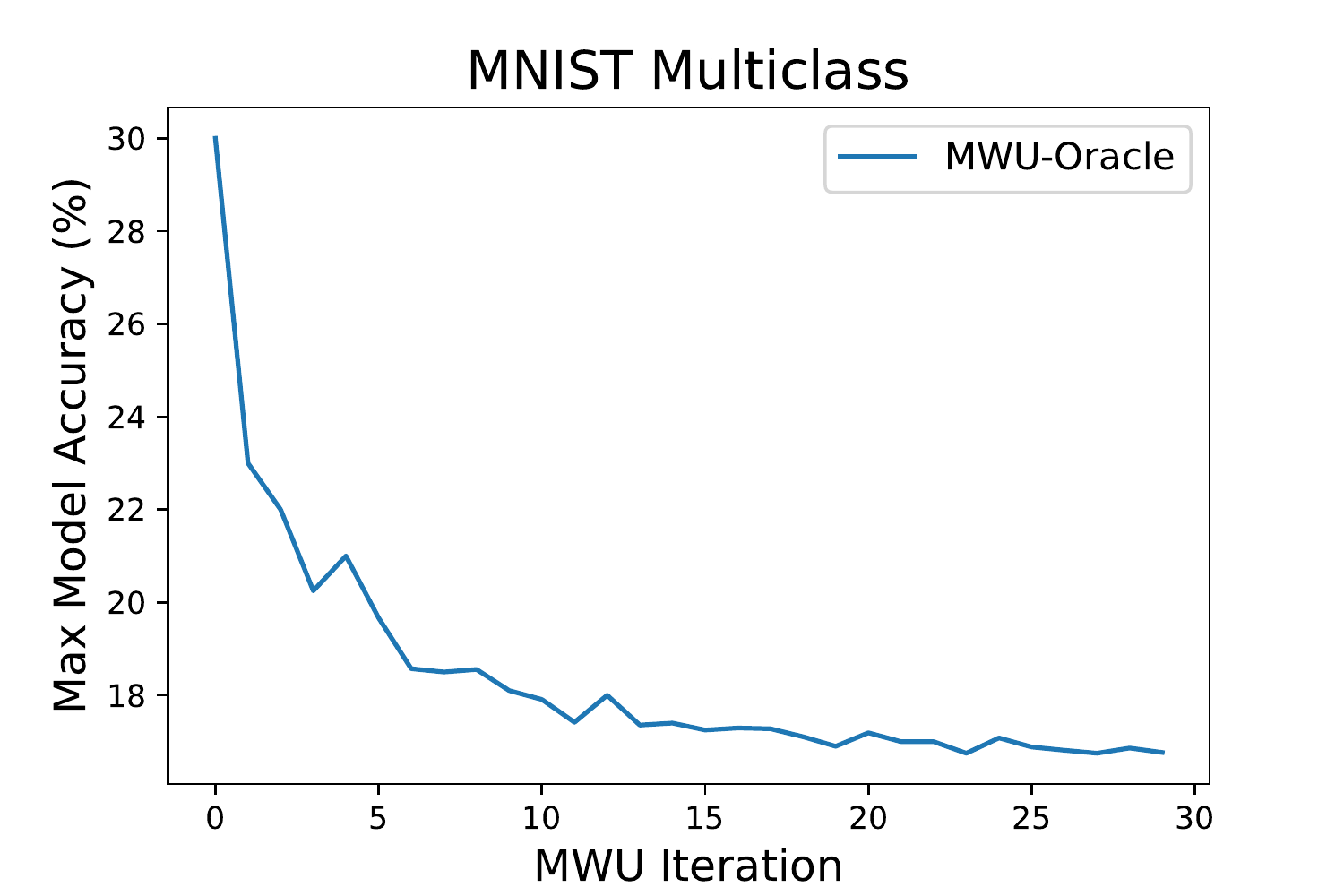}
\includegraphics[width=0.48\linewidth]{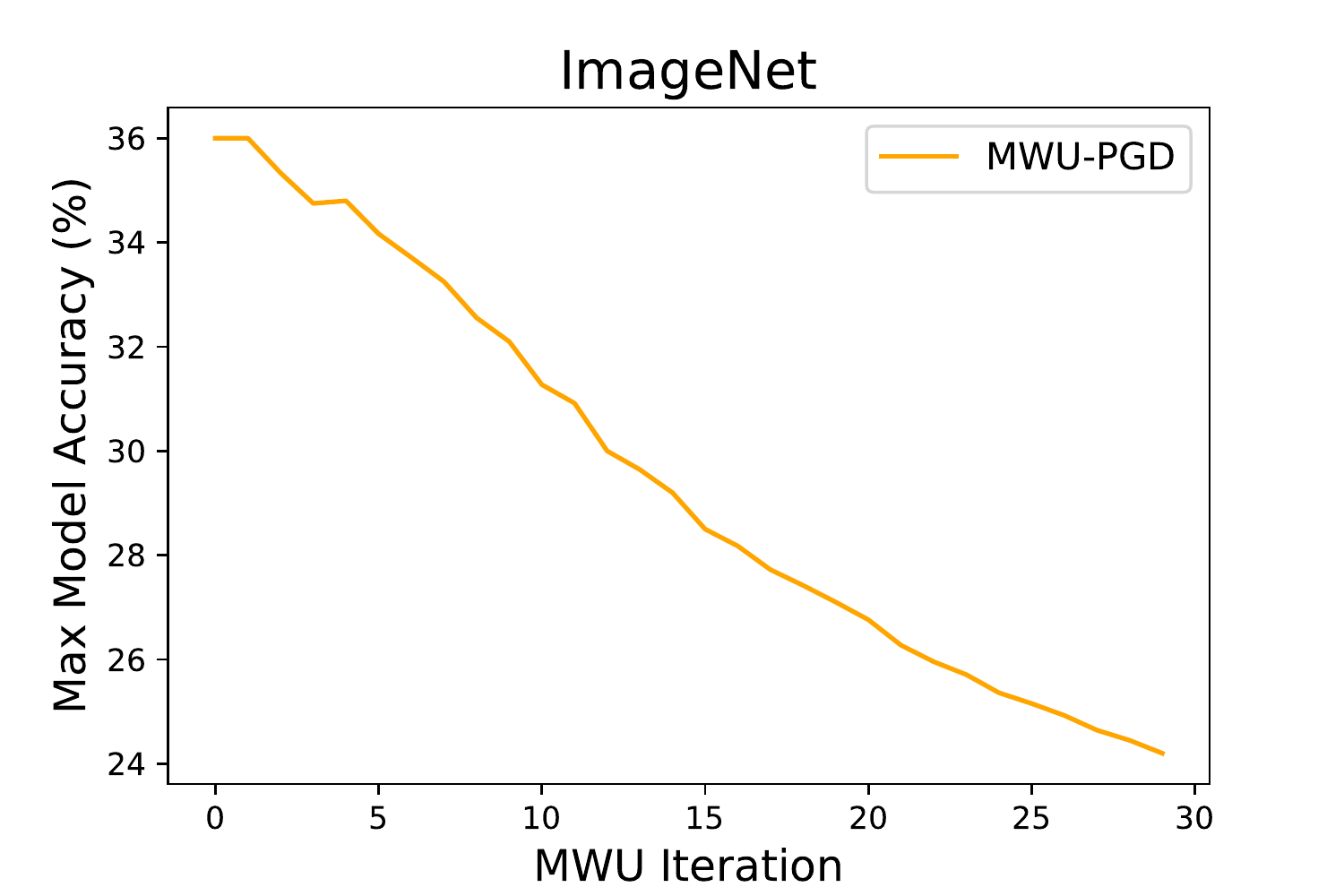}
\caption{Convergence of MWU for optimal attacks. For both MNIST multiclass models as well as neural networks on ImageNet, running MWU for several iterations significantly boosts the quality of adversarial noise. As seen in Table \ref{table:multi}, our methods significantly outperform other approaches.}
\label{fig:mwu_convergence}
\vspace{-12pt}
\end{figure*}

\begin{table}[b]
\vspace{-14pt}
\begin{center}
\caption{Results on linear multiclass models on MNIST (left) and deep networks on ImageNet (right). Entries describe the mean and max accuracies of classifiers under a particular noise algorithm.}
\vspace{14pt}
\label{table:multi}
\begin{small}
\begin{sc}
\begin{tabular}{lcccr}
\toprule
Noise Algorithm & Mean & Max\\
\midrule
Ensemble & 31.4\% & 55\%\\
Best Individual & 80\% &  100\% \\
Oracle & 12\% & 30\% \\
MWU - PGD & 34.6\% & 52\% \\
MWU - Oracle & \textbf{13.4\%} & \textbf{16.8}\%\\
\bottomrule
\end{tabular}
\end{sc}
\end{small}
\quad
\begin{small}
\begin{sc}
\begin{tabular}{lcccr}
\toprule
Noise Algorithm & Mean & Max \\
\midrule
Ensemble & 26.2\% & 55\% \\
Best Individual & 70.2\% & 99\% \\
Oracle & 16.2\% & 36\% \\
MWU - PGD & \textbf{15.12} \% & \textbf{24.2}\% \\
\bottomrule
\end{tabular}
\end{sc}
\end{small}
\end{center}
\end{table}

\textbf{Experimental setup.} For deep learning, we downloaded 5 pretrained ImageNet models with different architectures from the torchvision library: ResNet18, ResNet50, VGG13, VGG19 with batch norm, and DenseNet161.\footnote{Model accuracies may be found on the pytorch \href{https://pytorch.org/docs/stable/torchvision/models.html}{website}. Code for all experiments may be found \href{https://github.com/jcperdomo/robust_attacks/settings}{here}.} As before, we randomly select 100 images from the validation set that are correctly classified by all models. While we can no longer precisely calculate margins for each point, we choose noise budgets $\varepsilon$ so that the perturbed images remain visually indistinguishable from the originals as seen in  Appendix \ref{sec:more_experiments}. In particular, we perform our experiments $\varepsilon$ equal to 0.8.\footnote{If we divide $\varepsilon$ by the input dimension (224x224x3), this amounts to around $5\mathrm{\times}10^{-6}$ per channel.}

As discussed in Section \ref{sec:deep_learning}, we compute attacks against neural networks by running Algorithm \ref{alg:MWU} using the payout function $\Mut$. We approximate best responses for this modified loss by running PGD for 40 iterations on a weighted sum of untargeted reverse hinge losses, $\sum_{i=1}^n \p[i] \lut(c_i, x + v, y)$. As before, we clip solutions to the range $[0,1]$ so that they remain valid images.

For our baselines, we generate an ensemble classifier by computing an average over the logits of different individual models. To generate adversarial examples against the ensemble as well as for each individual model, we use the Momentum Iterative Method \cite{momentummethod} which won first place in the NIPS 2017 Adversarial Attacks Competition \cite{adversarialattackscomp}.  In addition to experimenting with the hyper parameters chosen by the authors in their original paper ($t\textrm{=}5$ iterations, decay factor $\mu\textrm{=}1$, and step size of $\varepsilon / T$), we also search over neighboring values and report the best results.
	
\textbf{Results.} Our results for deep neural networks mimic those of linear classifiers and further demonstrate how attacks developed for linear classifiers generalize well to deep learning. 
From Table \ref{table:multi}, we see that the gap between the  best response based methods we introduce and the other baselines is significant. As we motivate theoretically in Section \ref{sec:deep_learning}, approximating a best response using PGD on a weighted sum of untargeted reverse hinge losses results in a noise solution that significantly outperforms all baselines by a large margin. 

Running MWU-PGD for a single iteration, the \emph{Oracle} baseline, results in a maximum accuracy of 36\%. This demonstrates that, \emph{even without boosting} our best response oracle can be used to generate stronger attacks than existing approaches. If we do indeed boost noise using Multiplicative Weights, then we can further improve the quality of noise by an additional 12\% as seen in Figure \ref{fig:mwu_convergence}. Lastly, we highlight that the large gap between mean and max accuracies indicates that, in practice, there is a significant difference between fooling classifiers on average vs robustly minimizing the maximum accuracy. To properly attack a learner that randomizes across models, we must consider the latter.
\newpage
\bibliographystyle{abbrvnat}
\bibliography{attack_bib.bbl}

\newpage
\appendix

\begin{center}
\LARGE{\textbf{Supplementary material for \\
"Robust Attacks against Multiple Classifiers"}}
\end{center}
\begin{figure}[h]
\begin{tabular}{cccc}
\includegraphics[width=1.22in]{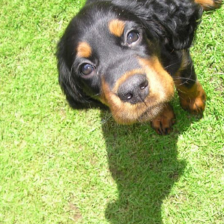} &
\includegraphics[width=1.22in]{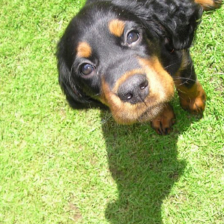} &
\includegraphics[width=1.22in]{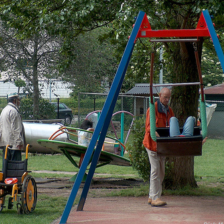} &
\includegraphics[width=1.22in]{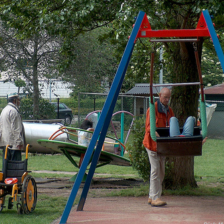} \\
\includegraphics[width=1.22in]{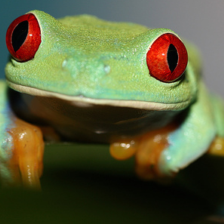} &
\includegraphics[width=1.22in]{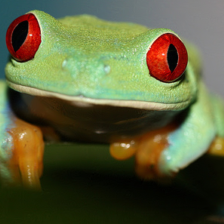} &
\includegraphics[width=1.22in]{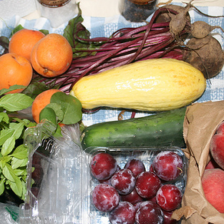} &
\includegraphics[width=1.22in]{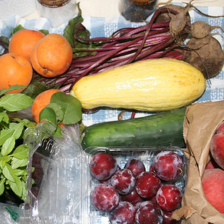} \\
\end{tabular}
\caption{Comparison of original and perturbed images from ImageNet under noise budget 0.8 in the $\ell_2$ norm. Perturbed images are on the right.}
\label{fig:pictures}
\end{figure}
\vspace{.2in}
\large{\textbf{Guide to the Appendix}}
\begin{normalsize}
\begin{itemize}
	\item In Appendix \ref{sec:more_experiments}, we present additional figures regarding the experiments on linear binary classifiers as well as further details on our experimental setup. 
	\item In Appendix \ref{appendix:ensembles}, we investigate, both empirically and theoretically, attacks on ensemble classifiers and demonstrate why they underperform our methods. 
	\item In Appendix \ref{sec:mwu_convergence}, we analyze the convergence of the Multiplicative Weights Update Algorithm as a means of approximating Nash equilibrium strategies. 
	\item In Appendix \ref{appendix:geometry}, we present the remaining proofs from our geometric characterization of best responses from Section \ref{sec:characterization}. In particular, we present proofs of Lemma \ref{lemma:convex_sets} and Theorem \ref{theorem:multi_oracle}.
	\item In Appendix \ref{appendix:hardness}, we prove that computing best responses is NP-hard.
	\item In Appendix \ref{appendix:pgd}, we prove Theorem \ref{theorem:revhinge}, and show how projected gradient descent on a relaxed version of the best response problem is guaranteed to return the optimal solution if a feasible misclassification set exists.
	\item Lastly, we present the proof of Corollary \ref{corollary:constant} in Appendix \ref{appendix:constant}
\end{itemize}

\section{Additional Experiments}
\label{sec:more_experiments}
\setcounter{page}{1}

Having presented experimental results for deep learning and linear multiclass models in the main body of the paper, in this section, we present our results on linear binary classifiers. Additionally, we include further details on our experimental setup.
\subsection{Further Details on Experimental Setup.}

In Figure \ref{figure:noise_budget_comparisons}, we progressively increase the amount of noise for each attack until it induces a maximum accuracy across classifiers that is comparable to the 16.8\% induced by MWU-Oracle for multiclass linear classifiers. Despite increasing the noise budget to 40 in the $\ell_2$ norm, the \emph{Best Individual} baseline did not manage to reduce the max accuracy below 99\%. This is explained by the fact that classifiers are sparse and have dissimilar decision boundaries as explained in the main body of the paper. For the \emph{Ensemble} and \emph{Oracle} baselines, similar levels of misclassification were achieved once we increased the noise budget to 2.1 and 1.57. Therefore, to achieve equal misclassification, we needed to use 62\% and 20\% more noise, respectively.

As discussed in Section \ref{sec:at_scale}, to guarantee convergence of MWU, we need to ensure that losses are constrained to the interval [0,1]. As previously defined, the reverse hinge loss can take values in $[0, \infty)$. However, when computing best responses, noise vectors must lie inside the Euclidean ball of radius $\varepsilon$. If we divide each individual loss by its maximum value over the $\ell_2$ ball, we can guarantee that losses are bounded to the interval $[0,1]$. The max loss is straightforward to calculate. Since pushing towards the boundary is the optimal attack, given a point $(x,y)$ the reverse hinge loss is maximized by pushing in the opposite direction (e.g $x +  \frac{\varepsilon y}{||w||_2}w$). Dividing by the max loss amounts to multiplying by a positive constant and thus preserves the convexity of the function.

\begin{figure}[t]
\label{fig:binary_convergence}
\begin{center}
\includegraphics[width=.6\linewidth]{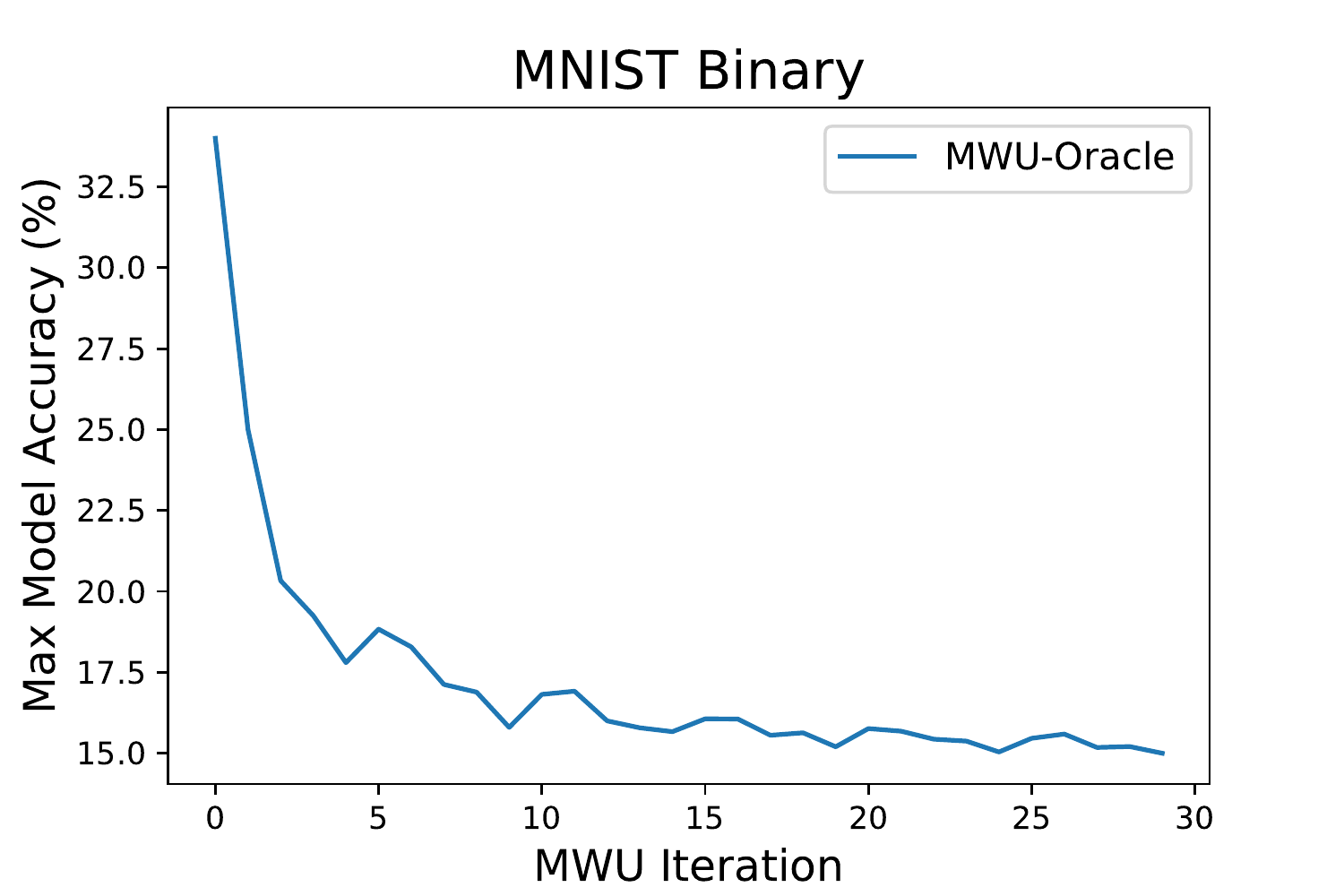}
\end{center}
\caption{Convergence of MWU on linear binary classifiers. Much like previous experiments for multiclass models and deep learning, MWU converges quickly and significantly boosts the quality of adversarial attacks.}
\end{figure}

To compute exact margins for linear classifiers, we use our exact oracle from Theorem \ref{theorem:multi_oracle}. In particular, given a single model $c\in\C$ and an example $(x,y)$, there are only $k-1$ regions $T_j$ where points are misclassified by the model. To compute the margin, the minimum distance from the point $x$ to the decision boundary, we can solve the convex program for the vector $v_j$ that pushes the point $x$ into each $T_j$, and compute the minimum length over all vectors $v_j$.

\begin{table}[b]
\vspace{-10pt}
\caption{Experimental results  on MNIST linear binary classifiers. As before, entries indicate the mean and maximum accuracy of classifiers in the set when evaluated on a particular attack.}
\vspace{5pt}
\label{table:binary_results}
\begin{center}
\begin{tabular}{lcccr}
\toprule
Noise Algorithm & Mean & Max\\
\midrule
Best Individual & 80\% &  100\% \\
Ensemble & 33.8\% & 65\%\\
Oracle & 12.8 \% & 34\% \\
MWU - PGD & 29.7\% & 42\% \\
MWU - Oracle & \textbf{13.6\%} & \textbf{15}\%\\
\bottomrule
\end{tabular}
\end{center}
\end{table}

When running projected gradient descent on a weighted sum of reverse hinge losses $f(v) = \sum_{i=1}^n \p[i] \lut(c_i, x + v, y)$, we compute updates according to the rule:
\begin{equation}
	v_{t+1} = \Pi_\varepsilon \Big(v_{t} - \eta \cdot \nabla f(v_{t}) / ||\nabla f(v_{t})||_2\Big)
\end{equation}

Where $\Pi_\varepsilon$ is the projection operator onto the $\ell_2$ ball of radius $\varepsilon$ and $\eta\mathrm{=}1.25\varepsilon / T$
. We experimented with using the traditional update rule $v_{t+1} = \Pi_\varepsilon(v_{t} - \eta \nabla f(v_{t}))$, for a small constant step size $\eta$ (e.g $\eta = 0.1$) and running for a larger number of iterations. However, we found that the practical performance between the different approaches was negligible and hence opted for the one which required fewer iterations. 

In addition to restricting vectors to the $\ell_2$ ball, as mentioned previously, we enforce box constraints so that each perturbed example remains a valid image. In particular, we clip iterates in PGD so that $x+v\in [0,1]^d$. In the case of the exact best response oracle, we augment the quadratic program with constraints of the form $0\leq x_i + v_i \leq 1$ for all $i \in [d]$.

For all our our experiments, we ran the MWU algorithm for $T\mathrm{=}30$ iterations and set the update parameter $\beta$ equal to $\sqrt{\ln |\C| / T}$ as indicated by our theoretical analysis presented in Appendix \ref{sec:mwu_convergence}.

\subsection{Experiment Results on Linear Binary Classifiers}

Similar to our experiments for multiclass classifiers, we find that MWU equipped with the best response oracle from Theorem \ref{theorem:multi_oracle} significantly outperforms all other baselines. Running the MWU algorithm for several iterations in the case of the exact oracle greatly improves the quality of the resulting noise solution. In particular, the gap between the \emph{Oracle} baseline and MWU-Oracle indicates that the maximum accuracy on the set of classifiers can be further reduced by an additional 20\% by considering the optimal randomized attack as seen in Table \ref{table:binary_results}. For MWU-PGD, contrary to the ImageNet case, we find that there is little benefit to boosting noise via MWU. 

\section{Why Ensembles Fail at White Box Attacks}
\label{appendix:ensembles}

When asked to find a noise solution that affects the performance of multiple classifiers, a natural approach one might consider is to attack the ensemble of all classifiers in that set. However, as we have seen previously, attacks that fool an ensemble do not always fool the underlying models. While ensembles have been shown to generate strong black box attacks \citep{song}, they seem to fail at generating robust white box attacks. In this section, we illustrate why this is the case.

\subsection{Understanding Ensemble Attacks Theoretically}

To understand the shortcomings of attacks on ensemble classifiers, we begin by theoretically characterizing their behavior on the simplest of settings: linear binary classifiers. Attacks on ensemble classifiers, typically consist of applying gradient based optimization methods to an ensemble model $E(\C, \p)$ made up of classifiers $c_i \in \C$ and ensemble weights $\p$. For binary classifiers, this ensemble classifier is computed by averaging the individual weight vectors as described in Section \ref{sec:experiments}, ($w_{ensemble}$$=\frac{1}{n}\sum_{i=1}^n w_i$). To find adversarial examples, we run gradient descent on a loss function such as the reverse hinge loss that is 0 if and only if the perturbed example $x' = x + v$ with true label $y$ is misclassified by the model.

Assuming $x'$ is not yet misclassified by the ensemble, the gradient of the loss function $\nabla \lr(E(\C, \p), x', y)$ is equal to $\sum_i \p[i] w_i$. This is undesirable for two main reasons:

\begin{itemize}
	\item First, the ensemble obscures valuable information about the underlying objective. If $x'$ is misclassified by a particular model $c_i$ but not the ensemble, contrary to applying PGD on a weighted sum of losses, $c_i$ still contributes $\p[i] w_i$ to the gradient and biases exploration away from promising regions of the search space;
	\item  Second, fooling the ensemble does not guarantee that the noise will transfer across the underlying models. Assuming the true label $y$ is -1, $\lr(E(\C, \p), x', y)=0$ if and only if there exists a subset $\T$ $\subseteq$ $\C$ such that: 
	\begin{equation}
		\sum_{c_t \in \T} \p[t](\<w_t, x' \> + b_t) > 0	\end{equation}
	\begin{equation}
		\sum_{c_j \in \C \setminus \T}\p[j](\<w_j,x'\> + b_j) < 0
	\end{equation}
	\begin{equation}
		\sum_{c_t \in \T} \p[t](\<w_t, x' \> + b_t) > \big|\sum_{c_j \in \C \setminus \T}\p[j](\<w_j,x'\> + b_j) \big|
	\end{equation}
	Hence, the strength of an ensemble classifier is only as good as its weakest weighted majority.
\end{itemize}

\subsection{Investigating Properties of Neural Network Ensembles using Saliency Maps}

Showing that attacking the ensemble is suboptimal in the linear case provides strong motivation as to why the method should perform poorly in the nonlinear case. However, to investigate this phenomenon further, we analyze how the decision boundaries differ between the individual classifiers and the ensemble network in the case of deep learning. Having different classification boundaries implies that attacks on one model are unlikely to affect other models as illustrated in Figure \ref{fig:randomization}.

We visualize the class boundaries of convolutional neural networks using the algorithm proposed by \citet{saliency} for generating \emph{saliency maps}. The class saliency map indicates which features (pixels) are most relevant in classifying an image to have a particular label. Therefore, they serve as one way of understanding the decision boundary of a particular model by highlighting which dimensions carry the highest weight. 

Given an input image $x$, they are defined as $\partial c_{i,j}(x) / \partial x$ where $c_{i,j}(x)$ is the $j$th logit of the network $c_i$. In the case of multichannel images, the value per pixel is defined as the maximum across all channels so as to yield a single grayscale image. Furthermore, we use the smoothed version of the saliency maps algorithm where derivatives are averaged over slightly perturbed inputs $x'= x+ v$ where $v$ is sampled from a zero-mean gaussian with small variance. 

\begin{figure}[t]
\begin{center}
	\includegraphics[width=.3\linewidth]{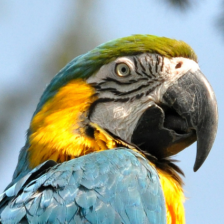}
\end{center}
\vspace{.1in}
\begin{tabular}{ccc}
\includegraphics[width=.3\linewidth]{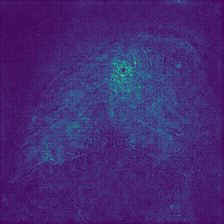} &
\includegraphics[width=.3\linewidth]{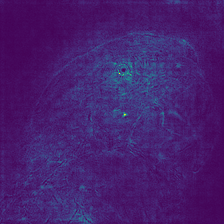} & 
\includegraphics[width=.3\linewidth]{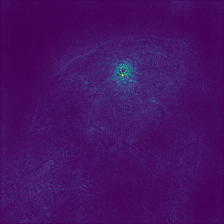} \\
\includegraphics[width=.3\linewidth]{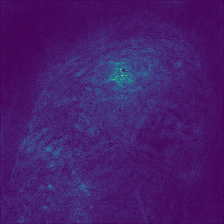} &
\includegraphics[width=.3\linewidth]{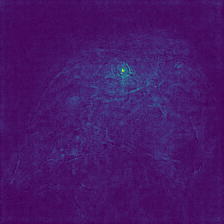} & 
\includegraphics[width=.3\linewidth]{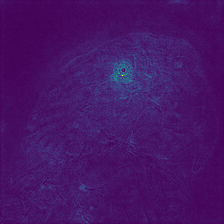} \\
\end{tabular}
\caption{Saliency Maps for ImageNet Classifiers. At the top is the true image. First row: ResNet18, ResNet50, VGG13. Second row: VGG19 with batch norm, DenseNet161, and the ensemble network. }
\label{fig:saliency_maps}
\end{figure}

In Figure \ref{fig:saliency_maps}, we see that the class saliency maps for individual models exhibit significant diversity. The ensemble of all 5 classifiers appears to contain information from all models, however, certain regions that are of central importance for individual models are relatively less prominent in the ensemble saliency map. This observation is in line with our earlier analysis of how ensemble classifiers obfuscate key information regarding the decision boundary of individual models. 

\section{Analysis of Multiplicative Weight Updates for Zero-Sum Games}
\label{sec:mwu_convergence}

\begin{theorem}
\label{theorem:mwu_convergence}
Given an error parameter $\delta$, after $\mathcal{O}(\frac{\ln n}{\delta^2})$ iterations, Algorithm \ref{alg:MWU} returns distributions $\p^\star$, $\q^\star$ s.t:
\begin{align*}
	\min_{c_i \in \C} M_\ell(c_i, \q^\star) & \geq \lambda - \delta \\
	 \underset{v \in \R^d, \; ||v||_2 \leq \varepsilon}{\max} M_\ell(\p^\star, v) & \leq \lambda + \delta
\end{align*}
where $\lambda = \underset{\q}{\min} \; \underset{\p}\max \; M_\ell(\p,\q)$ is the  value of the game .
\end{theorem}

\begin{proof}
The following analysis draws heavily upon the work of \citet{FreundSchapire, FreundSchapireGamePlaying}, yet the precise treatment follows that of \citet{kale}.

By guarantees of the Multiplicative Weights algorithm, we have that for any distribution $\p$ over $\C$ with losses in $[0,1]$, for $\beta \leq 1/2$ the following relationship holds (Corollary 1, \citet{kale}):
\begin{equation*}
\sum_{t=1}^T M_\ell(\p_t, v_t) \leq (1 + \beta)\sum_{i=1}^TM_\ell(\p, v_t) + \frac{\ln n}{\beta}
\end{equation*}
If we divide by $T$, and note that $M(\p, v) \leq 1$, and $M(\p_t, v_t) \geq \lambda$ for all $t$ (due to oracle guarantees), we have that for any distribution $\p$:
\begin{equation*}
	\lambda^\star  \leq \frac{1}{T} \sum_{i=1}^TM_\ell(\p_t, v_t) 
	  \leq \frac{1}{T}\sum_{i=1}^T M_\ell(\p, v_t) + \beta + \frac{\ln n}{\beta T}
\end{equation*}
If we let $\p = \tilde{\p}$ be the optimal strategy for the min player, then $M_\ell(\tilde{\p}, v) \leq \lambda^\star$ for any $v$. Furthermore, if we set $\beta = \frac{\delta}{2}$ and $T = \lceil \frac{4 \ln n}{\delta^2} \rceil$ we get that:
\begin{equation*}
\lambda^\star  \leq \frac{1}{T} \sum_{i=1}^TM_\ell(\p_t, v_t) \leq \frac{1}{T}\sum_{i=1}^T M_\ell(\tilde{\p},  v_t) + \delta \leq \lambda^\star + \delta
\label{eq:adversaryeq}
\end{equation*}
Therefore $\p^\star$, the uniform distribution over $\p_1, \dots, \p_t$ is an approximately optimal solution for the learner. 

For the adversary, we know from the previous equations that the following holds for any strategy $\p$ played by the learner:
\begin{equation*}
	\lambda^\star \leq \frac{1}{T} \sum_{i=1}^TM_\ell(\p_t, v_t) \leq \frac{1}{T}\sum_{i=1}^T M_\ell(\p, v_t) + \delta
\end{equation*}
If we set $\q^\star$ to be the distribution that assigns equal weight to  vectors in the set $\{v_1, \dots, v_T\}$ then we have that for any distribution $\p$:
\begin{equation*}
	\lambda^\star \leq \frac{1}{T} \sum_{i=1}^TM_\ell(\p_t, v_t) \leq M_\ell(\p, \q^\star) + \delta
\end{equation*}
Hence $\q^\star$ is an approximately optimal strategy:
\begin{equation*}
	\lambda^\star - \delta \leq M_\ell(\p, \q^\star)
\end{equation*}
\end{proof}

\section{Proofs of Geometric Characterization}
\label{appendix:geometry}

In this section, we present the proofs of our theoretical results from Section 3 that were omitted from the main body of the paper. In particular, we present the proofs of Lemma \ref{lemma:convex_sets} and Theorem \ref{theorem:multi_oracle}.

\begin{customlemma}{2}
Given a point, label pair $(x,y) \in \R^d \times [k]$ selecting a distribution $\p$ over a set $\C$ of $n$ linear classifiers, partitions the input space $\R^d$ into $k^n$ disjoint and convex sets $T_j$. Furthermore, $\R^d \setminus \bigcup_j T_j$ is a set of measure zero.
\end{customlemma}
\begin{proof}
From Lemma \ref{lemma:characterization}, we know that the sets are disjoint since two sets $T_j,T_{j'}$ must differ in at least one index and classifiers can only predict a single label for each point. To show that these sets are convex, consider points $x_1, x_2 \in T_j$ and an arbitrary classifier $c_i \in \C$ s.t. $c_i(x) = z$ for all $x \in T_j$. If we let $x' = \alpha x_1 + (1-\alpha )x_2$ where $\alpha \in [0,1]$ then the following holds for all $j \in [k]$ where $j \neq z$:
\begin{align*}
	c_{i,z}(x') & = \<w_{i,z},  \alpha x_1 + (1 - \alpha)x_2\> +  b_{i,z} \\
	& = \alpha \<w_{i,z},  x_1\> +   \alpha b_{i,z}  +  (1 - \alpha)\<w_{i,z},  x_2\> +  (1 - \alpha)b_{i,z}\\
	& > \alpha \<w_{i,j},  x_1\> +  \alpha b_{i,j}  +  (1 - \alpha)\<w_{i,j},  x_2\> +  (1 - \alpha)b_{i,j} \\
	& = c_{i,j}(x')
\end{align*}
Lastly, the set $\R^d \setminus \bigcup_i T_i$ is equal to the set of points $x$ where there are ties for the maximum valued classifier. This set is a subset of the set of points $\mathcal{K}$ that lie at the intersection of two hyperplanes:
\begin{equation}
	\R^d \setminus \bigcup_i T_i \subset \{x| \exists \; c_{i,k}, c_{j,l} \text{ s.t } c_{i,k}(x) = c_{j,l}(x)\}
\end{equation}
Finally, we argue that $\mathcal{K}$ has measure zero. For all $\varepsilon > 0, x \in \mathcal{K}$, there exists an $x'$ such that $||x - x'||_2 < \varepsilon$ and $x' \notin \mathcal{K}$ since the intersection of two distinct hyperplanes is of dimension two less than the overall space. Therefore, $\R^d \setminus \bigcup_i T_i$ must also have measure zero.
\end{proof}
\subsection{Best Response Oracle for Linear Classifiers}

For our analysis of Theorem \ref{theorem:multi_oracle}, we focus on the case where $\C$ consists of ``one-vs-all" classifiers. In the following subsection, we show how our results can be generalized to other methods for multilabel classification by reducing these other approaches to the ``one-vs-all" case. Given $k$ classes, a ``one-vs-all" classifier $c_i$ consists of $k$ linear functions $c_{i,j}(x) = \<w_{i,j}, x\> + b_{i,j}$ where $j \in [k]$. On input $x$, predictions are made according to the rule $c_i(x) = \argmax_j c_{i,j}(x)$.
\begin{customtheorem}{1}
For linear classifiers, implementing an exact \oracle reduces to the problem of minimizing a quadratic function over a set of $k^n$ convex polytopes.
\end{customtheorem}
\begin{proof}
From the previous lemmas, we know that the expected loss of the learner, $\Mz(\p,\cdot)$, can assume only finitely many values, each of which is associated with a particular convex region $T_j \subset \R^d$ . Therefore, to compute a best response, we can iterate over all regions and choose the  perturbation with $\ell_2$ norm less than $\varepsilon$ that is associated with the region of highest loss. To find points in a set $T_j$, each associated with label vector $s_j$, we solve for the vector $v$ of minimal $\ell_2$ norm such that $x+v \in T_j$. This can be done by minimizing a quadratic function over a convex set:
\begin{equation}
\label{eq:multi_prog}
\begin{aligned}
& \underset{v \in \R^d}{\text{min}}
& & ||v||_2^2 \\
& \text{subject to} & & c_1(x + v) = s_{j,1} \\
& & &. \dots\\
& & & c_n(x+ v) = s_{j,n}
\end{aligned}
\end{equation}
Each constraint in the program above can be expressed as $k-1$ linear inequalities. For a particular $z \in [k], c_i \in \C$ we write $c_i(x+v) = z$ as $c_{i,z}(x+ v) > c_{i,j}(x + v)$ for all $j \neq z$. 
\end{proof}
\vspace{-.2in}

\subsection{Beyond ``One-vs-All" Linear Classification}

Here we extend the results from our analysis of linear classifiers to other methods for multilabel classification. In particular, we show that any ``all-pairs" or multivector model can be converted to an equivalent ``one-vs-all" classifier and hence all of our results also apply to these other approaches.

\textbf{All-Pairs.} In the ``all-pairs" approach, each linear classifier $c$ consists of $\binom{k}{2}$ linear predictors $c_{i,j}$ trained to predict between labels $i,j \in [k]$. As per convention, we let $c_{i,j}(x) = -c_{j,i}(x)$. Labels are chosen according to the rule:
\begin{equation*}
\label{eq:all_pairs}
	c(x) = \argmax_{i \in [k]} \sum_{j \neq i} c_{i,j}(x)
\end{equation*}
Given an ``all-pairs" model $c$, we show how it can be transformed into a ``one-vs-all" model $c'$ such that $c(x)$=$c'(x)$ for all points $x \in \R^d$:
\begin{align*}
	 c(x) &= \argmax_{i \in [k]} \sum_{j \neq i} c_{i,j}(x) \\
	&= \argmax_{i \in [k]} \sum_{j \neq i} \<w_{i, j}, x \> + b_{i, j}\\
	&= \argmax_{i \in [k]} \<w'_i, x\> + b'_i\\
	&= \argmax_{i \in [k]} \sum_{j \neq i} c'_i(x) = c'(x)
\end{align*}
\textbf{Multivector.} Lastly, we extend our results to multilabel classification done via class-sensitive feature mappings and the multivector construction by again reducing to the ``one-vs-all" case. Given a function $\Psi: \R^d \times [k] \rightarrow \R^n$, labels are predicted according to the rule:
\begin{equation}
\label{eq:multivector_rule}
	c(x) = \argmax_{y \in [k]} \langle w, \Psi(x, y) \rangle
\end{equation}
While there are several choices for the $\Psi$, we focus on the most common, the multivector construction:
\begin{align*}
	\Psi(x,y) & = \big[ \underbrace{0, \dots, 0}_{\in \R^{(y-1)(d+1)}}, \underbrace{x_1, \dots, x_n, 1}_{\in \R^{d + 1}}, \underbrace{0, \dots, 0}_{\in \R^{(k-y)(d+1)}} \big] \\
	w & = \big[ w_1, \dots, w_k \big] \text{ where } w_i \in \R^{d+1} \; \forall i
\end{align*}
This in effect ensures that \eqref{eq:multivector_rule} becomes equivalent to that of the ``one-vs-all" approach:
\begin{equation*}
c(x) = \argmax_{i \in [k]} \langle w_i, x\rangle
\end{equation*}

\subsection{Extensions to $\ell_\infty$ norm}
\label{sec:ell_infty}

While we focus on the $\ell_2$ norm as the main metric with which to gauge the magnitude of adversarial noise, our results can be readily extended to function with the $\ell_\infty$ norm. Most of our results follow directly without modification, but for those that don't we present extensions here:

For linear classifiers, we can extend the result of Theorem \ref{theorem:multi_oracle} for the case of the $\ell_\infty$ norm by slightly modifying the convex program. Given a label vector $s_j$ and a point $(x, y)$ we solve for:
\begin{equation}
\label{eq:multi_prog_infty}
\begin{aligned}
& \underset{v \in \R^d}{\text{min}}
& & 0 \\
& \text{subject to} & & c_i(x + v) = s_{j,i} & \forall i \in [k] \\
& & & v_i \leq \varepsilon & \forall i \in [d]
\end{aligned}
\end{equation}
To extend our approximate best responses methods to the $\ell_\infty$ case, we can alter the projection step of gradient descent to constrain noise to the $\ell_\infty$ ball. The solution space remains convex and hence our theoretical guarantees still hold.

\section{Hardness of Computing a Best Response}
\label{appendix:hardness}
\begin{customtheorem}{2}
Given a set $\C$ of $n$ linear binary classifiers, a number $B$, a point $(x, y)$, noise budget $\varepsilon$, and a distribution $\p$, the problem of finding a vector $v$ with $||v||_2 \leq \varepsilon$ such that the loss of the learner $\Mz(\p, v) = B$ is NP-complete.
\end{customtheorem}
\begin{proof}
	We can certainly verify in polynomial time that a vector $v$ induces a loss of $B$ simply by calculating the 0-1 loss of each classifier. Therefore the problem is in NP.

	To show hardness, we reduce from Subset Sum. Given a set of $n$ numbers $P = \{p_1, \dots p_n\}$ and a target number $B$, the goal of Subset Sum is to find a subset $U \subseteq P$ such that the sum of the elements in $U$ equals $B$.\footnote{Without loss of generality, we can assume that instances of Subset Sum only have values in the range $[0,1]$. We can reduce from the more general case by simply normalizing inputs to lie in this range.} Given an instance of Subset Sum, we determine our input space to be $\R^n$, the point $x$ to be the origin, the label $y=-1$, and the noise budget $\varepsilon=1$. Next, we create $n$ binary classifiers of the form $c_i(x) = \< e_i, x \>$ where $e_i$ is the $i$th standard basis vector. We let $p_i$ be the probability with which the learner selects classifier $c_i$.\footnote{We can again normalize values so that they form a valid probability distribution.}

	We claim that there is a subset that sums to $B$ if and only if there exists a region $T_j \subset \R^n$ on which the learner achieves loss $B$. Given the parameters of the reduction, the loss of the learner is determined by the sum of the probability weights of classifiers $c_i$ such that $c_i(x+v) = +1$ for points $x + v \in T_j$. If we again identify sets $T_j$ with sign vectors $s_j \in \{\pm 1\}^n$ as per Lemma \ref{lemma:characterization}, there is a bijection between the sets $T_j$ and the power set of $\{p_1, \dots, p_n\}$. A number $p_i$ is in a subset $U_j$ if the $i$th entry of $s_j$ is equal to $+1$. 
	
	Lastly, we can check that there are feasible points within each set $T_j$ and hence that all subsets within the original Subset Sum instance are valid. Each $T_j$ corresponds to a quadrant of $\R^n$. For any $\varepsilon > 0$ and for any $T_j$, there exists a $v_j$ with $\ell_2$ norm less than $\varepsilon$ such that $x + v_j \in T_j$. Therefore, there is a subset $U_j$ that sums to $B$ if and only if there is a region $T_j$ in which the learner achieves loss equal to $B$.
\end{proof}

\section{Analysis of Projected Gradient Descent as a Best Response}
\label{appendix:pgd}

\begin{customtheorem}{3}
\label{theorem:revhinge}
	 For any noise budget $\varepsilon>0$, precision parameter $\beta>0$, and distribution $\p$ over $\C$, running  projected gradient descent for $\mathcal{O}(\varepsilon^2 / \beta^2)$ iterations on $f(v) = \sum_{i=1}^n \p[i] \lr(c_i, x + v, y)$ returns a solution $v_t$ such that $f(v_t) - f(v^\star) \leq \beta$, where $v^\star$ is the global minimum of $f$. Furthermore, if there exists a feasible misclassification set under $\varepsilon$, then it also holds that $f(v_t) - f(v_{BR}) \leq \beta$, where $v_{BR} = \textsc{best response}(\p, \varepsilon, \Mz)$ 
\end{customtheorem}

\begin{proof}
The reverse hinge loss is convex since it is the max of 0 and a linear function. The objective $f$ is thus also convex since it is a weighed sum of convex where all the weights are positive. In addition to being convex, the function is also Lipschitz. To do show that it is Lipschitz, since the function is convex, we only need to bound the norm of the gradient:
\begin{equation*}
	\begin{split}
	f(v_2) &\geq f(v_1) + \<\nabla f(v_1), v_2 - v_1\>\\
f(v_1) \textrm{-} f(v_2) &\leq \<\nabla f(v_1), v_1\textrm{-}v_2\> \leq ||\nabla f(v_1)||\cdot ||v_1 \textrm{-} v_2||
	\end{split}
\end{equation*}
Reversing the roles of $v_1$, $v_2$ we get that:
$$f(v_2) - f(v_1) \leq ||\nabla f(v_2)||\cdot ||v_2 - v_1||$$
Therefore $|f(v_1) - f(v_2)| \leq L\cdot ||v_1 - v_2||$ where $L$ is a bound on the norm of the gradient. The objective function $f(v) = \sum_{i=1}^n \p[i]\max \{y( \<w_i, x + v\> + b_i), 0\}$ has a max gradient of 
$$\sum_{i=1}^n \p[i]\cdot y \cdot w_i \leq \sum_{i=1}^n \p[i]  ||w_i||$$  
Since all the classifiers are just hyperplanes, we can normalize all the $w_i$ to have norm 1. Furthermore, since $\sum_{i=1}^n \p[i] = 1$, we get that the max norm of the gradient is 1. Hence the function is 1 Lipschitz.

We can now apply standard theorems for the convergence of projected gradient descent for convex Lipschitz functions. In particular, we use Theorem 3.2 from \citet{bubeck} which states that for convex, $L$-Lipschitz functions over the domain of a Euclidean ball with radius $R$, the following relationship holds with respect to the global optimimum, $v^\star$:
\begin{equation*}
f\Big(\frac{1}{T}\sum_{t=1}^T v_t\Big) - f(v^\star) \leq \frac{RL}{\sqrt{T}}
\end{equation*}
\end{proof}

To ensure that the average iterate is within $\beta$ of the optimum, setting $R=\varepsilon, L=1$ and solving for $T$, we get that $T$ must equal $\varepsilon^2 / \beta^2$. Furthermore, if a feasible misclassification set exists, then $f(v^\star) = f(v_{BR}) = 0$ and we get that the average iterate must be within $\beta$ of the optimal solution. 

\section{Efficiently Computing Optimal Attacks against Sets of Constant Size}
\label{appendix:constant}
\begin{customcorollary}{1}
When the number of linear classifiers is constant in the size of the input dimension $d$, the best response strategy can be computed in polynomial time.
\end{customcorollary}
\begin{proof}
When the number of classifiers is a small constant $r$, from Lemma \ref{lemma:characterization} it follows that there are only $k^r$ (polynomially many) regions $T_j$. As seen in Theorem \ref{theorem:multi_oracle}, to compute a best response, for each region, we solve a quadratic program over $d$ variables with $n(k-1)$ linear constraints. Since the run time of solving a quadratic program is polynomial in the number of variables and constraints \cite{quadprog}, and we are solving only polynomially many programs, the entire run time of computing a best response is polynomial. 
\end{proof}

\end{normalsize}

\end{document}